\documentclass[letterpaper]{article} 
\usepackage{aaai23}  
\usepackage{times}  
\usepackage{helvet}  
\usepackage{courier}  
\usepackage[hyphens]{url}  
\usepackage{graphicx} 
\usepackage{natbib}  
\usepackage{caption} 
\usepackage{array}
\usepackage{amsmath}
\usepackage{amssymb}
\usepackage{amsthm}
\usepackage{multirow}
\usepackage{booktabs}
\usepackage{subfigure}
\usepackage[title]{appendix}
\usepackage{algorithm}
\usepackage{algorithmic}

\usepackage{newfloat}
\usepackage{listings}
\newtheorem{lemma}{Lemma}
\newtheorem{theorem}{Theorem}

\newtheorem{definition}{Definition}
\DeclareCaptionStyle{ruled}{labelfont=normalfont,labelsep=colon,strut=off} 
\floatstyle{ruled}
\newfloat{listing}{tb}{lst}{}
\floatname{listing}{Listing}
\title{\textsc{FedGS}: Federated Graph-based Sampling with Arbitrary Client Availability}
\author{
Zheng Wang,\textsuperscript{\rm 1} 
Xiaoliang Fan, \textsuperscript{\rm 1,}\footnote{Corresponding Author} 
Jianzhong Qi, \textsuperscript{\rm 2} 
Haibing Jin,\textsuperscript{\rm 1}  
Peizhen Yang, \textsuperscript{\rm 1}\\ 
Siqi Shen,\textsuperscript{\rm 1} 
Cheng Wang\textsuperscript{\rm 1} 
}
\affiliations{
    \textsuperscript{\rm 1}Fujian Key Laboratory of Sensing and Computing for Smart Cities, School of Informatics, Xiamen University, Xiamen, China\\
    \textsuperscript{\rm 2}School of Computing and Information Systems, University of Melbourne, Melbourne, Australia\\

    zwang@stu.xmu.edu.cn, \\
    fanxiaoliang@xmu.edu.cn, jianzhong.qi@unimelb.edu.au, \{jinhaibing, yangpz\}@stu.xmu.edu.cn, \{siqishen,cwang\}@xmu.edu.cn
}

\usepackage{bibentry}
\begin{document}

\maketitle

\begin{abstract}
While federated learning has shown strong results in optimizing a machine learning model without direct access to the original data, 
its performance may be hindered by intermittent client availability which slows down the convergence and biases the final learned model. There are significant challenges to achieve both stable and bias-free training under arbitrary client availability. To address these challenges, we propose a framework named Federated Graph-based Sampling (\textsc{FedGS}), to stabilize the global model update and mitigate the long-term bias given arbitrary client availability simultaneously. First, we model the data correlations of clients with a Data-Distribution-Dependency Graph (3DG) that helps keep the sampled clients data apart from each other, which is theoretically shown to improve the approximation to the optimal model update. Second, constrained by the far-distance in data distribution of the sampled clients, we further minimize the variance of the numbers of times that the clients are sampled, to mitigate long-term bias. To validate the effectiveness of \textsc{FedGS}, we conduct experiments on three datasets under a comprehensive set of seven client availability modes. Our experimental results confirm \textsc{FedGS}'s advantage in both enabling a fair client-sampling scheme and improving the model performance under arbitrary client availability. Our code is available at \url{https://github.com/WwZzz/FedGS}.
\end{abstract}

\section{Introduction}

Federated learning (FL) enables various data owners to collaboratively train a model without sharing their own data \cite{mcmahan2017communication}. 
In a FL system, there is a server that broadcasts a global model to clients and then aggregates the local models from them to update the global model. Such a distributed optimization may cause prohibitive communication costs due to the unavailability of clients \cite{gu2021fast}.

As an early solution to this problem, \cite{mcmahan2017communication} propose to uniformly sample a random subset of clients without replacement to join the training process.
\begin{figure}
\setlength{\belowcaptionskip}{0cm}
    \centering
    \includegraphics[scale=0.34]{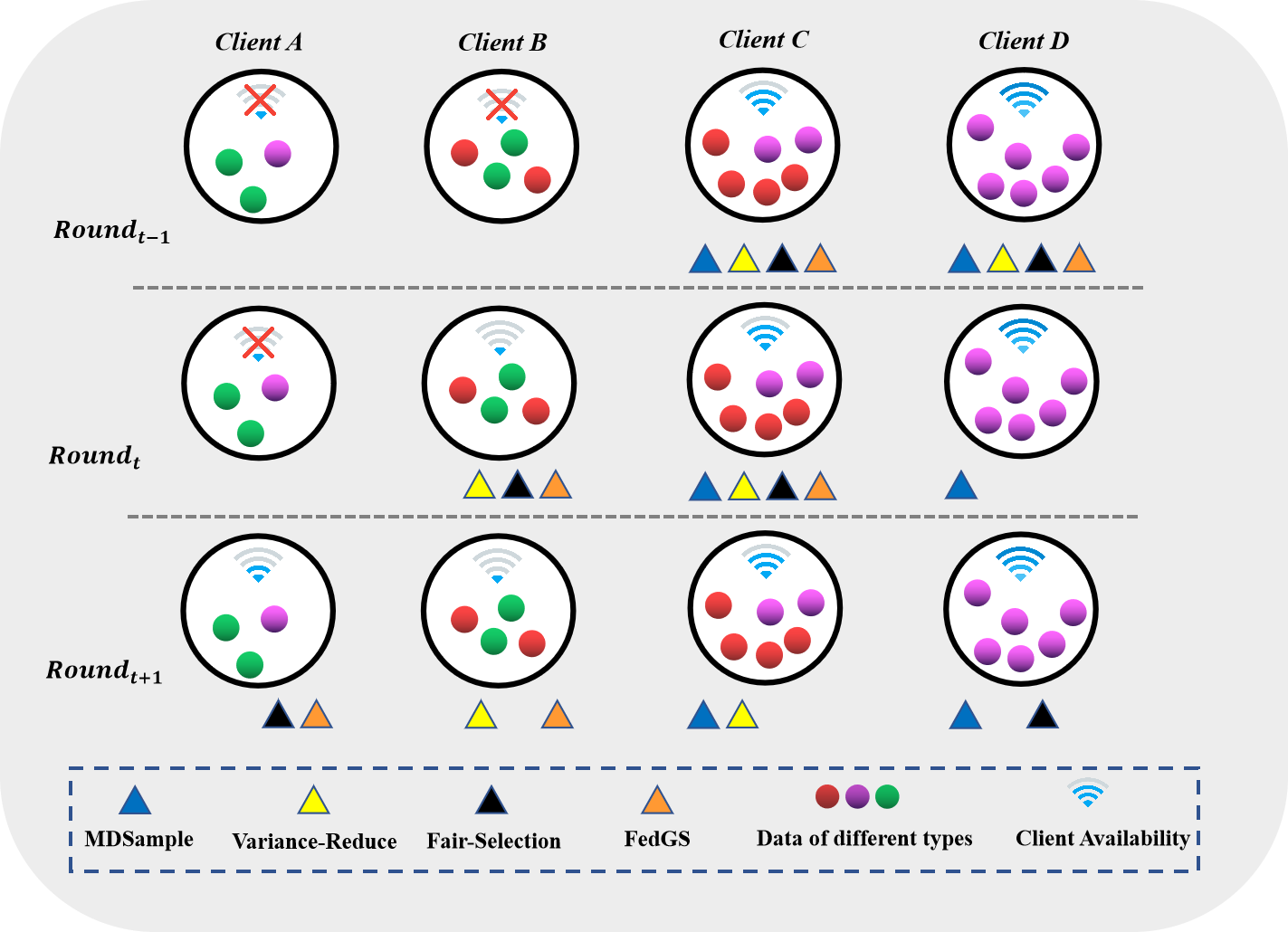}
    \caption{
An motivating example: there are significant challenging to achieve both stable model updates and bias-free training with the intermittent client availability in FL.
} 
    \label{fig1}
\end{figure}
\cite{li2020federated} sample clients in proportion to their data sizes with replacement to obtain an unbiased estimator of update. More recently, some works take the client availability into account when sampling clients \cite{yan2020distributed, gu2021fast, balakrishnan2021diverse, cho2020client, huang2020efficiency}. They show that selecting clients without considering whether the clients are active will lead to unbounded waiting times and poor responding rates. As a result,
they sample only the active clients to guarantee immediate client availability \cite{gu2021fast, cho2020client}. For example, in Fig.~\ref{fig1}, the server will not sample the inactive Client $A$ and Client $B$ at \textit{Round} ${t-1}$.

\textbf{However}, enabling both stable model updates and bias-free training under \emph{arbitrary client availability} (i.e., without any assumption on when each client will become available) poses significant challenges which have not been addressed. 
\textbf{On one hand}, the model is improved on the data distributions of the sampled clients at each round, which might lead to the detriment of the data specificity of non-sampled clients \cite{fraboni2021clustered}. For example, in Fig.~\ref{fig1}, fair-selection \cite{huang2020efficiency} tries to guarantee the least sampled times for each client (and hence it is ``fair''). While mitigating the long-term bias, it will ignore the data of the red type at \textit{Round} $t+1$, since it only considers the balance of the sampled frequency of clients. This fails to observe the data heterogeneity across the clients and leads to instability of the model update due to the absence of the gradients computed on the ``red'' data. 
%
%
\textbf{On the other hand}, the global models trained by FL may also be biased towards clients with higher availability in a long run. Also in Fig.~\ref{fig1}, the MDSample \cite{li2020federated} and Variance-Reduce \cite{fraboni2021clustered, balakrishnan2021diverse} methods, which do not consider the difference in client availability, introduce bias towards the clients with higher availability (i.e. Client $A$ is  overlooked at \textit{Round} $t+1$ regardless of not sampled in the previous rounds). In summary, there are significant challenges to address two competitive issues (i.e. stable model updates and bias-free training) that limit the FL training performance
under the arbitrary client availability.

To address the issues above, we propose a novel FL framework named \emph{\underline{Fed}erated \underline{G}raph-based \underline{S}ampling} (\textsc{FedGS}) to tackle the arbitrary client availability problem. We first model the data correlations of clients with a Data-Distribution-Dependency Graph (3DG) that helps keep the sampled clients data far from each other. We further minimize the variance of the numbers of times that the clients are sampled to mitigate long-term bias. Extensive experiments on three datasets under different client availability modes confirm \textsc{FedGS}'s advantage in both enabling a fair client-sampling scheme and improving the model performance under arbitrary client availability.

The contributions of this work are summarized as follow:
\begin{itemize}

\item We propose \textsc{FedGS} that could both stabilize the model update and mitigate long-term bias under arbitrary client availability. To the best of our knowledge, this is the first work that tackles the two issues simultaneously.


\item We propose the data correlations of clients with a Data-Distribution-Dependency Graph (3DG), which helps keep
sampled clients apart from each other, and is also dedicate to mitigate long-term bias. 

\item  We design a comprehensive set of seven client availability modes, on which we evaluate the effectiveness of \textsc{FedGS} on three datasets. We observe that \textsc{FedGS} outperforms existing methods in both client-selection fairness and model performance. 

\end{itemize}

\section{Background and Problem Formulation}
Given $N$ clients where the $k$th client has a local data size of $n_k$ and a local objective function $F_k(\cdot)$, we study the standard FL optimization problem as: 
\begin{align}
    \min_\theta F(\theta) =& \sum_{k=1}^N \frac{n_k}{n} F_k(\theta)
\end{align}
where $\theta$ is the shared model parameter and $n=\sum_{k=1}^N n_k$ is the total data size. 
A common approach to optimize this objective is to iteratively broadcast the global model (i.e., its learned parameter values) $\theta^t$ to all clients at each training round $t$ and aggregate local models $\{\theta^{t+1}_1,...,\theta^{t+1}_N \}$ that are locally trained by clients using SGD with fixed steps: 
\begin{align}
    \theta^{t+1} = \sum_{k=1}^{N}\frac{n_k}{n}\theta^{t+1}_k
\end{align}
When the number of clients is large, it is infeasible to update $\theta^{t+1}$ with $\theta^{t+1}_k$ from each client $k$, due to communication constraints. Sampling a random client subset $S_t \subset [N]$ to obtain an estimator of the full model update at each round becomes an attractive in this case, which is shown to enjoy convergence guarantee when the following unbiasedness condition is met \cite{li2019convergence, fraboni2021clustered}:
\begin{align}
\mathbb{E}_{S_t}\left[\theta^{t+1}\right]=\sum_{k=1}^{N}\frac{n_k}{n}\theta^{t+1}_k
\end{align}
Further, \citeauthor{fraboni2021clustered} and \citeauthor{balakrishnan2021diverse} propose to reduce the variance of the estimator as follows  to enable faster and more stable training: 
\begin{align}
Var(\nabla F_\theta) &= \|Var(\nabla F_\theta)\|_1 \\&= \mathbb{E}_t {\|\nabla F_{\theta^t} - \mathbb{E}[\nabla F_{\theta^t}]\|_2^2}
\end{align}
However, the effectiveness of these variance-reducing methods is still limited by the long-term bias caused by the arbitrary client availability as discussed earlier. 

\subsection{Mitigating long-term bias}
We first propose an objective that could mitigate the long-term bias without any assumption on the client availability. We denote the set of available clients at the round $t$ as $A_t\subseteq [N]$. Then, sampled clients should satisfy $S_t\subseteq A_t$ and $|S_t| \le M $, where $M$ is the maximum sample size limited by the server's capacity. To mitigate the impact of unexpected client availability on the sampled subset from a long-term view, we sample clients by minimizing the variance of the sampling counts of clients (i.e., the numbers of times that  the clients are sampled after $t$ rounds). Let the sampling counts of $N$ clients after $t$ rounds be $\bold{v}^t=\left[v_1^t, ..., v_N^t\right]$ where $v_k^t=\sum_{\tau=1}^{t}{\mathbb{I}(k\in S_{\tau})} = v_k^{t-1}+\mathbb{I}(k\in S_{t})$. Then, the variance of the client sampling counts after round $t$ is:
 \begin{align}
 &Var(\bold v^t)=\frac1{N-1}\sum_{k=1}^N(v_k^t-\bar{v}^{t})^2\\&=\frac{1}{N-1}\sum_{k=1}^N \left(v_k^{t-1} + \mathbb{I}(k \in S_t) - (\bar v^{t-1}+M/N)\right)^2
 \end{align}
 As discussed earlier, only balancing participating rates for clients may introduce large variance of model updates that slows down the model convergence \cite{fraboni2021clustered}. To enable a stable training, we introduce low-variance model updates as a constraint on the feasible space when minimizing the variance of sampling counts of the clients. We thus formulate our sampling optimization problem as:

\begin{align}
& \min_{|S_t|\le M, S_t\subseteq A_t}Var(\bold v^t) \\ & \text{s.t.} \quad\quad  Var(\nabla F_{S_t}(\theta^t))\le \sigma^2
\end{align}
where $\sigma^2 \ge 0$ is a coefficient that allows to search for a trade-off between the two objectives of stable model updates and balanced  client sampling counts. When $\sigma^2\rightarrow \infty$, the optimal solution will select the currently available clients with the lowest    sampling counts. On the other hand, a small $\sigma^2$ will limit the sampled clients to those with adequately small variance of the corresponding model updates.
\section{Methodology}
This section presents our solutions to the optimization problem above (Eq. 8 and 9). The main challenge lies in converting the constraint on the variance of the global model update into a solvable one. 
For this purpose, we utilize the data similarity between clients to increase the data diversity of the sampled subset. 
\begin{figure}
\setlength{\belowcaptionskip}{0cm}
    \centering
    \includegraphics[scale=0.45]{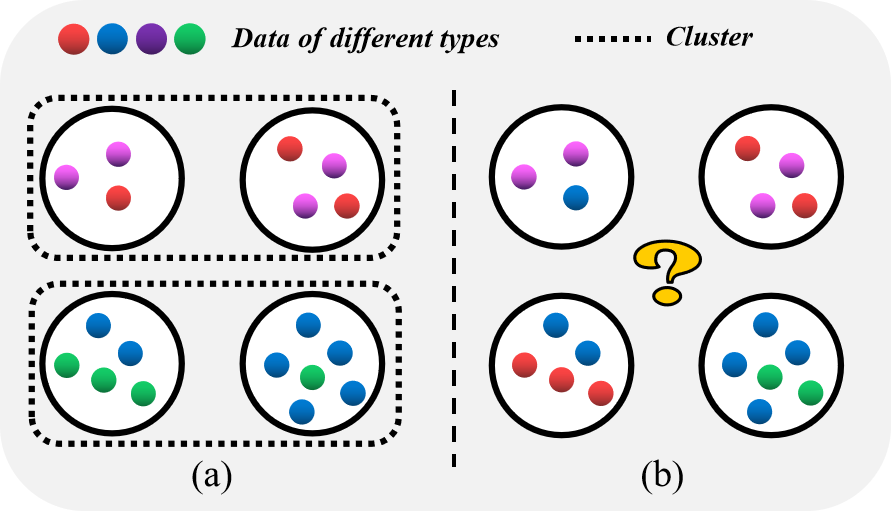}
    \caption{An example of two types of local
data distributions: (a) a simple cluster; and (b) a complex cluster.} \label{fig2}
\end{figure}
\subsection{Variance Reduction Based On 3DG}
Before illustrating our method, we briefly review previous works that address the Gradient Variance Reduction problem in FL. \citeauthor{li2020federated} add an proximal term to the local objectives to prevent the model from overfitting on the local data. \citeauthor{karimireddy2020scaffold} use control variate to dynamically correct the biased updates during local training. These two methods can avoid large model update variance by debiasing the local training procedure, which is orthogonal to the sampling strategy. \citeauthor{fraboni2021clustered} group the clients into $M$ clusters and then sample clients from these clusters without replacement to increase the chance for each client to be selected while still promising the unbiasedness of the model updates. Similarly, \citeauthor{balakrishnan2021diverse} approximates the full model updates by enlarging the diversity of the selected client set, which is done by minimizing a relaxed upper bound of the difference between the optimal full update and the approximated one. Such a relaxation aims to achieve that for each client $k\in [N]$ there exists an adequately similar client $i\in S_t$ in the sampled subset.

The existing methods work well when there are obvious clusters of clients based on their local data distributions in Fig.~\ref{fig2}(a). However, when the local data distributions are too complex to cluster like Fig.~\ref{fig2}(b), clustering the clients cannot accurately capture the implicit data correlations between clients, which may lead to performance degradation. Meanwhile, minimizing the relaxed upper is not the only means to enlarge the diversity of the sampled clients. We can achieve the same purpose without such minimization. 

To better describe the correlations among clients' local data distributions, we model the local data distribution similarities with a Data-Distribution-Dependency Graph (3DG) instead of grouping the clients into discrete clusters, as shown in Fig.~\ref{fig3}. Then, we show that keeping a large average shortest-path distance between the sampled nodes (i.e. clients) on the 3DG helps approximate the full model update. Intuitively, encouraging the sampled nodes to spread as far as possible helps differentiate the sampled local data distributions, which brings a higher probability to yield good balanced approximations for the full model update. This is proven as Theorem 1.


\begin{theorem}
Suppose that there are $C$ types of data (i.e., $C$ types of labels) over all datasets, where each data type's ratio is $p_i$ such that the dataset can be represented by the vector $\bold p^*=[p_1^*,...,p_C^*],\bold 1^\top \bold p^*=1$. Without losing generality, consider $\bold p^*$ to be uniformly distributed in the simplex in $\mathbb{R}^C$, the number of local updates to be $1$, and 3DG is a complete graph. A larger distance of sampled clients on the 3DG leads to a more approximate full model update.
\end{theorem}
\begin{proof}
See Appendix A.
\end{proof}

Although the proof is based on that 3DG is a complete graph, we empirically show that keeping the clients far away from each others on the 3DG can benefit FL training even when this assumption is broken.

\begin{figure}
\setlength{\belowcaptionskip}{0cm}
    \centering
    \includegraphics[scale=0.16]{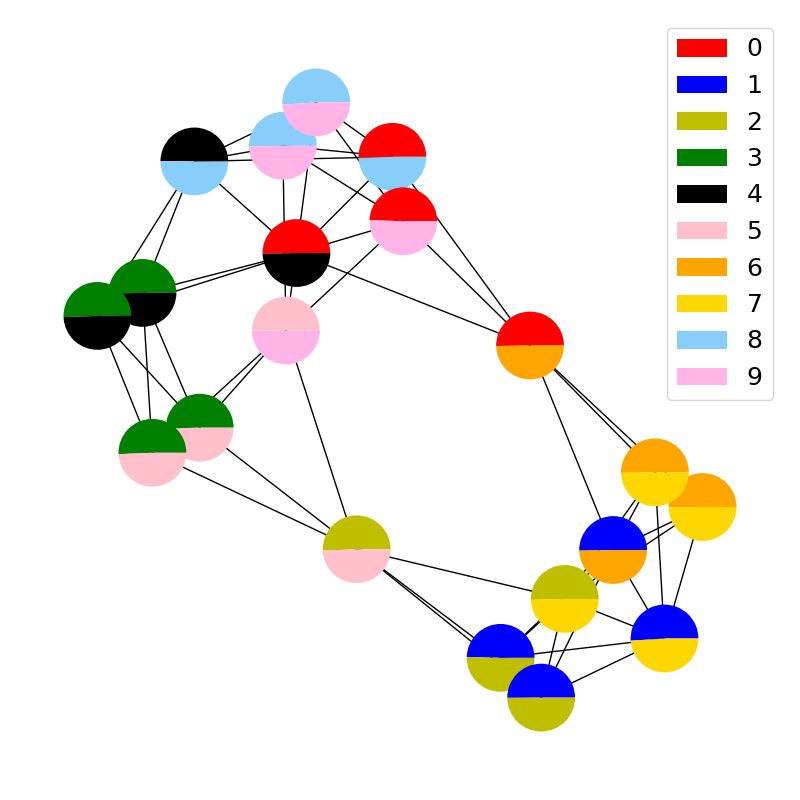}
    \caption{A visualized example of the oracle 3DG generated on CIFAR10 partitioned by 20 clients and each of them owns data with only two labels. The different color means the ratio of corresponding labels in their local dataset.} \label{fig3}
\end{figure}
\subsection{Construction of 3DG in FL}
Now we discuss how to construct the 3DG. A straightforward approach is to directly calculate the distance (e.g., KL divergence) between different local data distributions, which is infeasible in FL because the clients do not share their local data. Without loss of generality, we assume that there is a feature vector $\bold u_k \in \mathbb{R}^{d}$ that can well represent the information about the local data distribution of each client $c_k$. We argue that this is achievable in practice. For example, training an ML model for tasks of supervised learning (e.g., classification) usually face severe data heterogeneity in FL, where there may exist label skewness in clients' local data (e.g. each client only owns data with a subset of labels). In this case, the label distribution vectors (i.e. the number of items of each label) can well reflect the bias of each client's local data. Once given the feature vectors $\bold U = [\bold u_1, \bold u_2,...,\bold u_N]$, we can easily calculate the similarity between any two clients $c_i$ and $c_j$ with a similarity function $f_{sim}:\mathbb{R}^{d}\times \mathbb{R}^{d}\rightarrow [0,1]$  as: 
\begin{align}
    \bold V = [V_{ij}]_{N\times N}, V_{ij} = f_{sim}(\bold u_i, \bold u_j)
\end{align}
Then, the similarity matrix $ \bold{V} $ can be converted into an adjacent matrix $ \bold{R} $ for the 3DG over the clients for the 3DG over the clients by:

\begin{align*}
\begin{split}
R_{ij} = \left \{
\begin{array}{ll}
    0,&i=j\\
    e^{-V_{ij}/\sigma^2}&i\ne j,V_{ij}\ge\epsilon \\
    \infty,&i\ne j,V_{ij}<\epsilon
\end{array}
\right.
\end{split}
\end{align*}

where $\epsilon>0$ is a positive threshold used to control the sparsity of the adjacent matrix and $\sigma$ controls the diversity of the edge weight. A large value of $\sigma$ will lead to small difference between the edge weights.

The feature vector $\bold u_k$ can also leak sensitive information about the clients, and it may not be exposed to the other clients or the server. It is necessary for the server to obtain the similarities between clients in  privacy-preserving manner to reconstruct or accurately approximate the \emph{oracle 3DG} (i.e., the 3DG corresponding to the true features $\bold U$). To achieve this goal, we present two methods that can help the server construct the 3DG.

The first is to use techniques based on Secure Scalar Product Protocols (SSPP) \cite{wang2009toward, shundong2021secure}, which aims to compute the dot product of private vectors of parties. We argue that any existing solutions of SSPP can be used to reconstruct the similarity matrix $\bold V$ of clients in our settings, and we detail one of the existing solutions for scalar product protocal \cite{du2002building} with a discussion on how to adapt it to build the 3DG in the Appendix D.

Although the SSPP-based methods can construct the oracle 3DG without any error, it cannot be adapted to the case where the feature vectors are difficult or impossible to obtain. In addition, the SSPP-based method only applies when the similarity function $f_{sim}$ is a simple dot product. 

We propose a second method that computes the similarity between clients based on their uploaded model parameters. Given the locally trained models $\theta_i^{t+1}$ and $\theta_j^{t+1}$, a straightforward way to calculate the similarity between the two clients is to compute the cosine similarity of their model updates \cite{xu2020reputation, xu2021gradient}
\begin{align}
 V_{ij} = \max(\frac{\Delta \theta_i^{t\top}\Delta \theta_j^t}{\|\Delta \theta_i^t\|\|\Delta \theta_j^t\|},0)
\end{align}
where $\Delta \theta_i^{t\top}=\theta_i^{t+1} - \theta^t$.
 However, since the model parameters' update is usually of extremely high dimensions but low rank \cite{azam2021recycling}, the direct cos similarity may contain too much noise, causing inaccuracy when constructing the 3DG. Motivated by \cite{baek2022personalized}, we instead compute the functional similarity based on the model parameters  to overcome the problem above. We first feed a batch of random Gaussian noise $\epsilon\sim\mathcal{N}(\mu,\Sigma)$ to all the locally trained models, where $\mu$ and $\Sigma$ are respectively the mean and covariance of a small validation dataset owned by the server \cite{zhao2018federated}. Then, we take the average of the $l$th layer's network embedding on this batch for each client to obtain $e_i$, and we compute the  similarity as:
\begin{align}
e_i = \theta_i(\epsilon)[l], V_{ij} =\max(\frac{e_i^\top e_j}{\|e_i\|\|e_j\|},0)
\end{align}
where we set $l$ as the output layer in practice. 

Another concerning issue is that the server may not have access to all the clients' feature vectors during the initial phase. As a result, the adjacent matrix of clients may need to be dynamically built and polished round by round. Nevertheless, we emphasize that we are not trying to answer how to optimally capture the correlations between clients' local data distributions in FL. Instead, we aim at showing that the topological correlations of clients' local data can be utilized to improve the training process of FL, and we key how to build the optimal 3DG for as our future work. 

For convenience, we simply assume that all the clients are available at the initial phase, by which the server can obtain the 3DG just before training starts. We empirically show the effectiveness of the approximated 3DG in Sec. 4.4.

\subsection{\textsc{FedGS}}
We now present our proposed Federated Graph-based Sampling (\textsc{FedGS}) method. As mentioned in Sec. 3.1, we bound the variance of the global model update at each round by keeping a larger average shortest-path distance between each pair of sampled clients. Given a 3DG, we first use the Floyd–Warshall algorithm to compute the shortest-path distance matrix $\bold H=[h_{ij}]_{N\times N}$ for all pairs of clients. Let $s_k^t\in\{0,1\}$ be a binary variable, where $s_k^t=1$ means client $c_k$ is selected to participate training round $t$ and $s_k^t=0$ otherwise. Then, the sampling result in round $t$ can be $\bold s_t=[s_1^t,...,s_N^t]\in \{0,1\}^N$, where the average shortest-path distance between sampled clients is written as:     
\begin{align}
  g(S_t)= \frac{2}{N(N-1)}\sum_{i,j\in S_t, i\ne j}h_{ij}s_i^t s_j^t = \frac{\bold s_t^{\top}{\bold H}\bold s_t}{N(N-1)}
\end{align}
 Accordingly, we replace the constraint in Equation (9) with $g(S_t)\ge \alpha$, and we convert this constraint into a penalty term added into the objective. After rewriting the equation (8) to be a maximization problem based on $s_t$, we obtain: 
 
\begin{align}
\max_{ \bold s_t\le \bold a^t, s_k^t \in \{0,1\}}\frac{\alpha \bold s_t^{\top}{\bold H}\bold s_t }{N(N-1)}- \frac 1{N-1}\bold{z}^\top \bold s_t
\\ \text{s.t. }\quad\quad\bold 1^\top\bold s_t=min(M, |A_t|)
\end{align}

where $z_k=2(v_k^{t-1}-\bar v^{t-1}-M/N)+1$, $\bold a^t= \{a_1^t, ...,a_N^t\}\in \{0,1\}^N $ and $a_k^t=1$ means client $c_k$ is available in round $t$. Note that $s_k^t=0$ for the clients unavailable clients in round $t$ and $s_k^{t2}=s_k^t$. Thus, Equation (14) can be reduced to: 
\begin{align}
\max_{\tilde{s}_k^{t} \in \{0,1\}}\bold{ \tilde{s}}_t^{\top}\left(\frac{\alpha}{N}\tilde{\bold H} - \textbf{diag}(\tilde{\bold{z}})\right)\bold{ \tilde{s}}_t\\
\text{s.t.       }\quad\bold 1^\top\bold{ \tilde{s}}_t=min(M, |A_t|)
\end{align}
$\bold{ \tilde{s}}_t=[s_{i1}^{t},...,s_{i|A_t|}^{t}]\in \{0,1\}^{|A_t|}$ and $s_{ij}^{t}=1$ represents that the $j$th client in the available set is selected. $\tilde{\bold{z}}\in \mathbb{R}^{|A_t|}$ and $\tilde{\bold H}\in \mathbb{R}^{|A_t|\times |A_t|}$ also only contains the element where the corresponding clients are available in round $t$.

This rewritten problem is a constrained mixed integer quadratic problem, which is a variety of an NP-hard problem, \textit{p-dispersion} \cite{pisinger1999exact}, with a non-zero diagonal. We optimize it to select clients within a fixed upper bound of wall-clock time. We empirically show that a local optimal can already bring non-trivial improvement when the client availability varies. 
\subsubsection{Aggregation Weight.}
Instead of directly averaging the uploaded model parameters like \cite{balakrishnan2021diverse, li2020federated}, We normalize the ratio of the local data size of selected clients as weights of the model aggregation:
\begin{align}
    \theta^{t+1} = \sum_{k\in S_t}\frac{n_k}{\sum_{i\in S_t}n_i} \theta^{t+1}_k
\end{align}
We argue that this is reasonable in our sampling scheme. Firstly, \textsc{FedGS} forces to balance the sampling counts of all the clients regardless of their availability. Thus, for convenience, we simply assume that all the clients will be uniformly sampled with the same frequency $\frac{MT_c}{N}$ in every $T_c$ rounds, and that the size of the set of available clients $|A_t|$ is always larger than the sample size limit $M$ in each round $t$. By treating the   frequency $\frac{MT_c}{N}/T_c=M/N$ as the probability of each client being selected without replacement in each round $T_0+\tau$, $(\tau\le T_c)$, we obtain:
\begin{align}
\mathbb{E}_{S_t}[\theta_{t+1}] &= \mathbb{E}_{S_t}\left[\frac MN\sum_{k=1 }^N\frac{n_k}{n_k+\sigma(S_{t},k)} \theta^{t+1}_k\right]\\&=\frac MN\sum_{k=1 }^N\frac{n_k}{n_k+\sigma_k} \theta^{t+1}_k
\end{align}
where $\sigma_k=\mathbb{E}_{S_t}[\sigma(S_{t},k)]=\mathbb{E}_{S_t}[\sum_{j\in S_t, j\ne k}n_j] = \frac{M-1}{N-1}(n-n_k)$. Therefore, the expected updated model of the next round follows:
\begin{align}
\mathbb{E}_{S_t}[\theta_{t+1}] &= \frac MN\sum_{k=1 }^N\frac{n_k}{n_k+\frac{M-1}{N-1}(n-n_k)} \theta^{t+1}_k\\
&=\sum_{k=1 }^N\frac{n_k}{n}\frac{1}{1+\frac1M \frac{N-M}{N-1}\frac{n_k-\bar n}{\bar n}} \theta^{t+1}_k\\
&=\sum_{k=1 }^N\frac{n_k}{n}\gamma_k \theta^{t+1}_k
\end{align}

From Equation (22), we see that the degree of data imbalance will impact the unbiasedness of the estimation. When the data size is balanced as $n_k=\bar n,\forall k\in[N]$, the estimation is unbiased since $\gamma_k=1,\forall k\in[N]$. If the data size is imbalanced, the degree of data imbalance will only have a controllable influence on the unbiasedness with the ratio of each client's local data's size to the average data size $\frac{\|n_k-\bar n\|}{\bar n}$. This impact can be 
immediately reduced by increasing the number of sampled clients $M$ at each round. 

The analysis above is based on the assumption that our proposed \textsc{FedGS} can well approximate the results obtained by uniform sampling without replacement in ideal settings. Generally speaking, a small $Var(\bold v^t)$ will limit the difference between the sampling counts, which is also empirically verified by our experimental results. The pseudo codes in Algorithm 1 summarizes the main steps of \textsc{FedGS}.
\begin{algorithm}[tb]
    \caption{Federated Graph-Based Sampling}
    \label{alg:1}

    \textbf{Input}:The global model $\theta$, the feature matrix of clients' data distribution $\bold U$, the maximum wall-clock time of the solver $\tau_{max}$, the sizes of clients' local data $n_k$, the number of local updating steps $E$, and the learning rate $\eta_t$\\
    \begin{algorithmic}[1]
    \STATE Initialize the global model parameters $\theta_0$ and the sampling counts of clients $\bold v^0=[0,...,0]\in \mathbb{N}^{N}$.
    \STATE Create the 3DG $\bold G$ based on the techniques in Sec. 3.2.
    \STATE Compute the shortest-path distance of each pair of nodes on 3DG by Floyd Algorithm to obtain $\bold H$.
    
    \FOR{ communication round $t = 0,1,...,T-1$}
    \STATE The server checks the set of available clients $A_t$.
    \STATE The server uses $\bold v^t$ and $\bold H$ to solve equation (16) within the maximum wall-clock time $\tau_{max}$ to obtain the sampled client set $S_t\subseteq A_t$
    \STATE The server broadcasts the model $\theta^t$ to clients in $S_t$.
    \FOR{ each client $k\in S_t$}
    \FOR{ each iteration $i=0,1,...,E-1$}
    \STATE $\theta_{k,i+1}^{t}\leftarrow \theta_{k,i}^{t}-\eta_t\nabla F_k(\theta_{k,i}^{t})$ 
    \ENDFOR
    \STATE Client $k$ send the model parameters $\theta_{k}^{t+1}=\theta_{k,E}^{t}$ to the server.
    \ENDFOR
    \STATE The server aggregates the received local model parameters $    \theta^{t+1} = \sum_{k\in S_t}\frac{n_k}{\sum_{i\in S_t}n_i} \theta^{t+1}_k$ 
    \STATE The server updates the    sampling counts of clients $\bold v^{t+1}[k]\leftarrow \bold v^{t}[k] + \mathbb{I}(k\in S_t)$
    \ENDFOR
    \end{algorithmic}
    \end{algorithm}
    
\begin{table*}[t]
\normalsize
\renewcommand{\arraystretch}{1.7}  
\centering
\resizebox{\linewidth}{!}{  
\begin{tabular}{c|c|c|c|c|c} 
\hline
\multirow{2}{*}{\textbf{Name}} & \multirow{2}{*}{\textbf{Description}} & \multicolumn{3}{c|}{\textbf{Dependency}} & \multirow{2}{*}{\textbf{Active Rate}}  \\ 
\cline{3-5}
 & & \textbf{Time} & \textbf{Data} & \textbf{Other} &        \\ 
 \hline
\textbf{ID}ea\textbf{L} & Full client availability & - & - &  - &  $1$   \\
\hline
 \textbf{M}ore\textbf{D}ata\textbf{F}irst (Ours)  & More data, higher availability & - & data size $n_k$ &  -         &  $\frac{n_k^\beta}{\max_i n_i^\beta}$   \\
 \hline
 \textbf{L}ess\textbf{D}ata\textbf{F}irst (Ours) & Less data, higher availability  &-& data size $n_k$ & -          &  $\frac{n_k^{-\beta}}{\max_i n_i^{-\beta}}$        \\
 \hline
  \textbf{YM}ax\textbf{F}irst \cite{gu2021fast}     & Larger value of label, higher availability  &-& value of label set $\{y_{ki}\}$  &  -& $\beta\frac{\min_i\{y_{ki}\}}{\max_{c,j}\{y_{cj}\}}+(1-\beta)$                         \\
 \hline
  \textbf{YC}ycle (Ours)   & Periodic availability with label values   &round $t$ &value of label set $\{y_{ki}\}$&-& $\beta\mathbb{I}\left( \cup_{y_{ki}}\frac{1+(t\%T_p)}{T_p}\in [\frac{y_{ki}}{num_Y},\frac{y_{ki}+1}{num_Y})\right)+(1-\beta)$   \\
 \hline
  \textbf{L}og \textbf{N}ormal \cite{ribero2022federated}  &   Independent availability obeying \textit{lognormal}&-&-&$c_k\sim\text{lognormal}(0,ln\frac{1}{1-\beta})$ & $\frac{c_k}{\max_i c_i} $   \\
 \hline
  \textbf{S}in \textbf{L}og \textbf{N}ormal \cite{ribero2022federated} &\textit{Sin}-like intermittent availability of   with \textbf{LN} &round $t$  &  -  & $c_k\sim\text{lognormal}(0, ln\frac{1}{1-\beta})$ &  $\frac{c_k}{\max_i c_i}\left(0.4sin\left(\frac{1+(t\%T_p)}{T_p}2\pi\right)+0.5\right)    $  \\
  \hline
\end{tabular}
}
\caption{An Overview of Different Client Availability Modes.}
\label{table1}
\end{table*}

\begin{table*}\scriptsize
\renewcommand{\arraystretch}{1}  
\centering
\begin{tabular}{l|l|l|l|l|l|l|l|l|l|l|l|l|l} 
\hline
\textbf{Dataset}         & \multicolumn{5}{c|}{\textbf{ Synthetic(0.5, 0.5)}}                                                                                                                               & \multicolumn{5}{c|}{\textbf{CIFAR10}}                                                                                                                                            & \multicolumn{3}{c}{\textbf{FashionMNIST}}                            \\ 
\hline
\textbf{Availability}    & \multicolumn{1}{c|}{\textbf{IDL}} & \multicolumn{1}{c|}{\textbf{LN}} & \multicolumn{1}{c|}{\textbf{SLN}} & \multicolumn{1}{c|}{\textbf{LDF}} & \multicolumn{1}{c|}{\textbf{MDF}} & \multicolumn{1}{c|}{\textbf{IDL}} & \multicolumn{1}{c|}{\textbf{LN}} & \multicolumn{1}{c|}{\textbf{SLN}} & \multicolumn{1}{c|}{\textbf{LDF}} & \multicolumn{1}{c|}{\textbf{MDF}} & \multicolumn{1}{c|}{\textbf{IDL}}& \multicolumn{1}{c|}{\textbf{YMF}} & \multicolumn{1}{c}{\textbf{YC}}  \\ 
\hline
\textbf{UniformSample}   &0.302&0.320&0.324&0.330&0.362&0.975&1.042&1.038&1.049&0.999&0.315&0.331&0.333\\ 
\hline
\textbf{MDSample}        &0.302& 0.322 &0.328&0.328&0.326 &0.971&1.037&1.048&1.051&0.991&0.315&0.333&0.338\\ 
\hline
\textbf{Power-of-Choice} &0.691&0.362&0.352&0.557&\textbf{0.301}&1.287&1.108&1.078&1.267&1.026&0.345&0.326&0.311\\  
\hline
\textbf{\textsc{FedProx} $\mu=0.01$} &\textbf{0.301}&0.346&0.376&0.319&0.410&0.972&1.056&1.162&1.039&0.995&0.315&0.331&0.374\\ 
\hline
\textbf{\textsc{FedGS} $\alpha=0.0$} &0.307&\textbf{0.305}&0.320&0.309&\underline{0.310}&1.006&0.976&1.002&0.974&0.967&0.302& 0.310&0.324\\ 
\hline
\textbf{\textsc{FedGS} $\alpha=0.5$} &0.311 & 0.306&0.319& \textbf{0.308}&0.311&0.977&0.973&1.000&\textbf{0.966}&0.974&\textbf{0.299}&0.312&0.312\\ 
\hline
\textbf{\textsc{FedGS} $\alpha=1.0$} &0.328 &0.306&\textbf{0.318}&\textbf{0.308}&0.311&\textbf{0.968}&\textbf{0.972}&\textbf{0.996}&0.971&\textbf{0.963}&0.300& 0.308& 0.310\\ 
\hline
\textbf{\textsc{FedGS} $\alpha=2.0$} &\underline{0.306}&0.307&0.320&\textbf{0.308}&0.311&0.970&0.975&1.001&0.973&0.969&0.310&\textbf{0.303}&0.312\\ 
\hline
\textbf{\textsc{FedGS} $\alpha=5.0$} &0.317 & 0.307&0.321&0.309&0.311&0.976&0.974&0.996&0.972&0.972&0.307&\textbf{0.303}&\textbf{0.307}\\ 
\hline
\end{tabular}
\caption{The optimal testing loss of methods running under different client availability modes on three datasets. Each result in the table is averaged over 3 different random seeds.}
\label{table2}
\end{table*}

\section{Experiment}
\subsection{Experimental Setup}
\subsubsection{Datasets and models to be trained}
We validate \textsc{FedGS} on three commonly used federated datasets: \textbf{Synthetic (0.5, 0.5)} \cite{li2020federated}, \textbf{CIFAR10} \cite{krizhevsky2009learning} and \textbf{FashionMNIST} \cite{xiao2017fashion}. For Syntetic dataset, we follow the settings use by \cite{li2020federated} to generate imbalance and non-i.i.d. dataset with 30 clients. 
For CIFAR10, we unequally partition the dataset into 100 clients following the label distribution  $Y_k\sim Dir(\alpha p)$ \cite{measure} ($p$ is the global label distribution).
For FashionMNIST, we balance the data sizes for 100 clients, each of whom owns data of only two labels. We train a logistic regression model for Synthetic(0.5, 0,5) and CNN models for CIFAR10 and FashionMNIST. More details on datasets are in Appendix C.
\subsubsection{Client Availability}
%
We first review the client availability settings discussed in existing FL literature \cite{ribero2022federated, gu2021fast}, a common way is to allocate an active probability to each client at each round. We observe that the client's active probability may depend on data distribution or time \cite{ribero2022federated, gu2021fast}.
We mainly conclude the existing client availability modes and propose a comprehensive set of seven client availability modes in Table \ref{table1} to conduct experiments under arbitrary availability. For each mode, we set a coefficient $\beta\in [0,1]$ to control the degree of the global unavailability, where a large value of $\beta$ suggests a small chance for most devices to be active. A further explanation about these availability modes is provided in Appendix C, where we also visualize the active state of clients at each round.
\subsubsection{Baselines}
We compare our method \textsc{FedGS} with: (1) \textsc{UniformSample} \cite{mcmahan2017communication}, which samples available clients uniformly without replacement, (2) \textsc{FedProx}/\textsc{MDSample} \cite{li2020federated}, which samples available clients with a probability proportion to their local data size and trains w/wo proximal term, (3) \textsc{Power-of-Choice} \cite{cho2020client}, which samples available clients with top-$M$ highest loss on local data and is robust to the client unavailability. Particularly, all the reported results of our \textsc{FedGS} are directly based on the oracle 3DG. We put results obtained by running \textsc{FedGS} on the constructed 3DG in the Appendix B.

\subsubsection{Hyper-parameters}
For each dataset, we tune the hyper-parameters by grid search with FedAvg, and we adopt the optimal parameters on the validation dataset of FedAvg to all the methods. The batch size is $B=10$ for Synthetic and $B=32$ for both CIFAR10 and FashionMNIST. The optimal parameters for Synthetic, Cifar10, FashionMNIST are resepctively $\eta=0.1,E=10$, $E=10, \eta=0.03$ and $E=10,\eta=0.1$. We round-wisely decay the learning srate by a factor of $0.998$ for all the datasets.
More details about the hyper-parameters are put in Appendix C.

\subsubsection{Implementation} All our experiments are run on a 64 GB-RAM Ubuntu 18.04.6 server with Intel(R) Xeon(R) CPU E5-2630 v4 @ 2.20GHz and 4 NVidia(R) 2080Ti GPUs. All code is implemented in PyTorch 1.12.0.

\subsection{Results of Impact of Client Availability}
We run experiments under different client availability modes to study the impact of arbitrary client availability, where the main results are shown in Table 2. Overall, the optimal model performance measured by the test loss is impacted by the client availability modes for all methods and over all three datasets. On Synthetic, UniformSample suffers the worst model performance degradation, $19.8\%$ (i.e. \textbf{IDL} v.s. \textbf{MDF}), while that for MDSample is $8.6\%$ (i.e. \textbf{IDL} v.s. \textbf{LDF}). \textsc{FedGS} retrains a strong model performance, with a degradation no more than $5\%$ for all the values of $\alpha$. Further, our proposed \textsc{FedGS} achieves the best performance under all the availability modes except for \textbf{IDL} and  \textbf{MDF}. For \textbf{IDL}, our \textsc{FedGS} still yields competitive results comparing with UniformSample and MDSample (0.306 v.s. 0.302). For \textbf{MDF}, although Power-of-Choice achieves the optimal result, its model performance is extremely unstable across different availability modes, e.g., its model performance is almost $200\%$ worse than the others for \textbf{IDL}. Meanwhile, \textsc{FedGS} still achieves $4.9\%$ improvement over MDSample and $14.3\%$ over UniformSample in this case. For CIFAR10 and FashionMNIST, all the optimal results fall into the region runned with \textsc{FedGS}. Using $\alpha>0$ brings a non-trivial improvement over using $\alpha=0$ in most of the client availability modes, which suggests that the variance reduction of \textsc{FedGS} can also benefit FL training.

\begin{figure}
\setlength{\abovecaptionskip}{0cm}
\setlength{\belowcaptionskip}{0cm}
    \centering
    \includegraphics[scale=0.34, trim=0 0 0 0, clip]{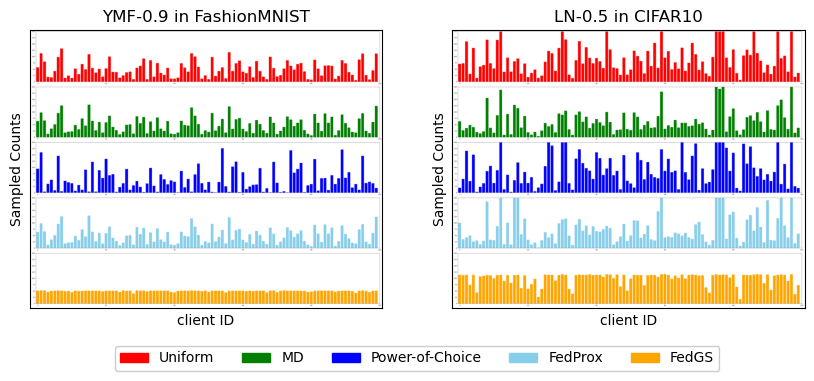}
    \caption{The results of final client sampling counts on FashionMNIST-YMF-0.9 and Cifar10-LN-0.5.} \label{fig4}
\end{figure}

\begin{figure}
\setlength{\abovecaptionskip}{0cm}
\setlength{\belowcaptionskip}{0cm}
\centering
    {%
        \includegraphics[width = .48\linewidth]{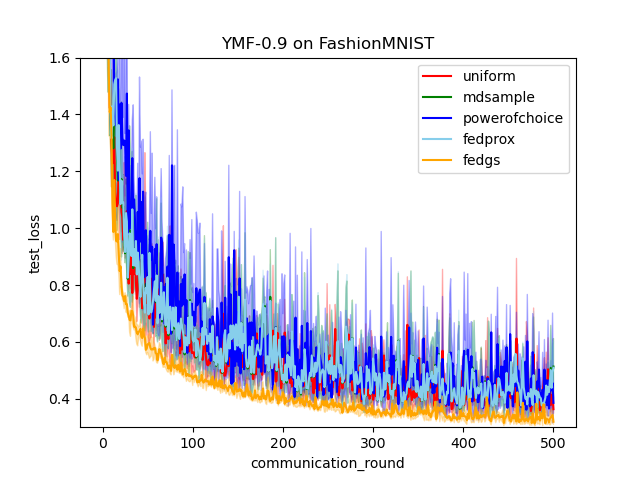}
        \includegraphics[width = .48\linewidth]{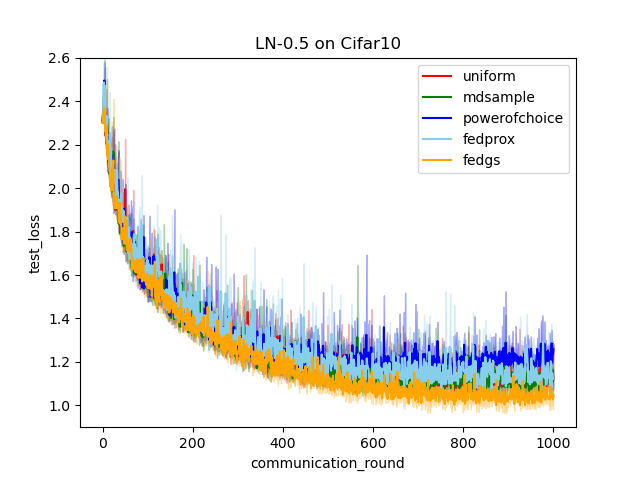}}

\caption{The testing loss curve respectively on FashionMNIST-YMF-0.9 and CIFAR10-LN-0.5.\label{fig:5}}
\end{figure}
\subsection{Results of Client Selection Fairness}

As shown in Fig.~\ref{fig4}, $\textsc{FedGS}_{\alpha=1}$ can substantially enhance fairness in client selection, which results in a uniform distribution of client sample counts both on FashionMNIST with YMF-0.9 and CIFAR10 with LN-0.5. We also show the corresponding curves of test loss in Fig.~\ref{fig:5}, which show that \textsc{FedGS} can stabilize FL training and find a better solution.
\subsection{Effectiveness of 3DG Construction}
    To validate the effectiveness of our method to reconstruct the oracle 3DG based on the functional similarity of model parameters, we use the F1-score of predicting the edges in the oracle 3DG to measure the quality of the construction, and we compare results with those obtained using the cosine similarity. The oracle 3DG is generated with $\epsilon=0.1$ and $\sigma^2=0.01$. Considering the difference in the feature space of the oracle and those of the model-based methods, we vary the value of $\epsilon\in\{0,0.01,0.05,0.1,0.5\}$ and report the results with the highest F1-score for each method. Results in Table 3 confirm the effectiveness of the proposed method to approximate the oracle 3DG, where the functional-similarity method achieves a higher F1-score than the cos-similarity method on both datasets. The results of \textsc{FedGS} running on the 
   model-based 3DG are included in Appendix B.

\begin{table}\scriptsize
\setlength{\belowcaptionskip}{0cm}
\centering
\begin{tabular}{l|l|l|l|l} 
\hline
\textbf{Dataset}   & \textbf{Method}   & \textbf{Precision} & \textbf{Recall} & \textbf{F1-Score}  \\ 
\hline
\multirow{2}{*}{\textbf{CIFAR10}} & \textit{functional similarity} & 0.8789  & \textbf{0.8700}  & \textbf{0.8744} \\ 
\cline{2-5}
                                  & \textit{cosine similarity}        & \textbf{1.0000} & 0.1316  & 0.2327              \\ 
\hline
\multirow{2}{*}{\textbf{FashionMNIST}} & \textit{functional similarity} & \textbf{0.9761}  & 0.7097  & \textbf{0.8218}              \\ 
\cline{2-5}
                                  & \textit{cosine similarity}   & 0.3765  & \textbf{0.9853} & 0.5448             \\
\hline
\end{tabular}
\caption{The effectiveness of how to construct 3DG.}
\label{table3}
\end{table}

\section{Related Works}
\subsection{Client Sampling in FL}
Client sampling is proven to has a significant impact on the stability of the learned model \cite{cho2020client}. \cite{mcmahan2017communication} uniformly samples clients without replacement to save communication efficiency. \cite{li2020federated} samples clients proportion to their local data size and uniformly aggregate the models with full client availability. \cite{fraboni2021clustered} reduces the variance of model to accelerate the training process. Nevertheless, these works ignored the long-term bias introduced by arbitrary client availability, which will result in the model overfitting on a particular data subset. Recent works \cite{ribero2022federated, gu2021fast, huang2020efficiency} are aware of such long-term bias from the arbitrary availability of clients. However, these two competitive issues (e.g. stable model updates and bias-free training) have not been considered simultaneously.


\subsection{Graph Construction in FL}
When it is probable to define the topology structure among clients in FL, several works directly utilized underlying correlations among different clients according to their social relations \cite{he2021fedgraphnn}. Other works proposed to connect the clients with their spatial-temporal relations as a graph \cite{meng2021cross, zhang2021fastgnn}. However, those works are used to conduct explicit correlations between clients (e.g. social relation, spatial relation), which were not able to uncover the important and implicit connections among clients. In short, we are the first to construct Data-Distribution-Dependency Graph (3DG) to learn the potential data dependency of sampled clients, which is proven to both guarantee a fair client sampling scheme and improve the model performance under arbitrary client availability.

\section{Conclusion}
We addressed the long-term bias and the stability of model updates simultaneously to enable faster and more stable FL under arbitrary client availability. To this end, we proposed the \textsc{FedGS} framework that models clients' data correlations with a Data-Distribution-Dependency
Graph (3DG) and utilizes the graph to stabilize the model updates. To mitigate the long-term bias, we minimize the variance of the numbers of times that clients are sampled under the far-distance-on-3DG constraint. Our experimental results on three real datasets under a comprehensive set of seven client availability modes confirm the robustness of \textsc{FedGS} on arbitrary client availability modes. In the future, we plan to study how to define and construct 3DG across various ML tasks.

\section{Acknowledgements}
The research was supported by Natural Science Foundation of China (62272403, 61872306), Fundamental Research Funds for the Central Universities (20720200031), FuXiaQuan National Independent Innovation Demonstration Zone Collaborative Innovation Platform (No.3502ZCQXT2021003), and Open Fund of PDL (WDZC20215250113).
\bibliography{aaai23}

\appendix
\begin{appendix}

\section{A. Proof of Theorem 1}
\begin{proof}
Supposing there are $C$ types of data in the global dataset, the global dataset can be represented by the data ratio vector $\bold p^*=[p_1^*, p_2^*,...,p_C^*]^\top$, where each data type's ratio is $p_i^*> 0$ and $\bold 1^\top \bold p^*=1 $. Then, the local data distribution of each client $c_k$ can also be represented by $\bold p_k=[p_{k1},p_{k2},...,p_{kC}]^\top,\bold 1^\top \bold p_k=1 $.
We slightly modify the way of constructing 3DG as
\begin{align*}
\begin{split}
R_{ij} = \left \{
\begin{array}{ll}
    0,&i=j\\
    \|e_i-e_j\|_2^2&i\ne j
\end{array}
\right.
\end{split}
\end{align*}
where $e_i=\frac{p_i}{\|p_i\|}, e_j=\frac{p_j}{\|p_j\|}$ and the principle of the smaller similarity corresponding to the larger distance still holds. Given that 3DG is a complete graph, we demonstrate that the shortest-path distance matrix $\bold H$ is the same to $\bold R$, since $distance(i,k)+distance(k,j)\ge distance(i,j)$ is always established for any $k$.


In each communication round $t$, the server samples $M$ clients to participate. By denoting the sampled clients' normalized data distribution as $\bold P_{t}=[\bold e_{i_1},...,\bold e_{i_M}]\in \mathbb{R}_+^{C\times M}$, the average shortest-path distance of these selected clients is 
\begin{align}
    f(S_t)&=\frac{1}{M(M-1)}\sum_{i,j\in S_t}H_{ij} \\&= \frac{1}{M(M-1)}\left(\sum_{i,j\in S_t}(2-2e_i^\top e_j)\right)\\&=\frac{2M}{M-1}-\frac{2}{M(M-1)}\|\bold P_t^\top\bold P_t\|_{m1}
\end{align}

After receiving the local models, the server aggregates their uploaded models to obtain
\begin{align}
    \theta_{t+1}=\sum_{k\in S_t }w_k \theta^{t+1}_k=\theta_t +\sum_{k\in S_t}w_k\Delta\theta_k^{t}
\end{align}
where $w_k$ is the aggregation weight of the selected client $c_k$ and $\bold 1^\top \bold w=1$. Now, the question is how possible can the aggregated model update recover the true model update $ \Delta \theta_{t}$ that is computed on all the clients?

To answer this question, we first simply consider the model update computed on all the $i$th type of data to be $\Delta\bar{\theta}_{ti}\in\mathbb{R}^d$ and $\bold\Delta_t=[\Delta\bar{\theta}_{t1},...,\Delta\bar{\theta}_{tC}]\in \mathbb{R}^{d\times C}$, then we can approximate the true model update by
\begin{align}
    \Delta \theta_t^*&=\sum_{i=1}^C p_i^* \Delta\bar{\theta}_{ti}\\&=\bold\Delta_t\bold p*
\end{align}
Similarly, each client $c_k$'s model update can also be represented by 
\begin{align}
    \Delta\theta_k^t=\sum_{i=1}^C p_{ki} \Delta\bar{\theta}_{ti}=\bold\Delta_t \bold p_k
\end{align}
Then, the aggregated model update is
\begin{align}
    \sum_{k\in S_t}w_k\Delta\theta_k^{t}&=\sum_{k\in S_t}\bold\Delta_t \frac{\bold p_k }{\|\bold p_k \|}w_k\|\bold p_k \|
    \\&=\bold\Delta_t \bold P_t (\bold{w}\odot [\|\bold p_1\|,...,\|\bold p_M\|]^\top)
        \\&=\bold\Delta_t \bold P_t \bold{\tilde w}
\end{align}
When there exists an optimal weight vector $\bold w^*$ such that $\exists \mu>0, \mu\bold P_t \bold w^*=\bold p^*$, we demonstrate that the sampled subset can well approximate the true model update, which requires the true weight $\bold p^*$ to fall into the region of the convex cone $\mu\bold P_t \bold w^*$. Given the assumption that $p^*$ is uniformly distributed in the simplex of $\mathbb{R}^C$, the problem becomes how to compare the probability  $\textit{Pr}\left(\bold p^*\in\{\mu\bold P_t \bold w^*|\bold 1^\top\bold w=1, \bold w\ge0, \mu>0\}\right)$ for different $\bold P_t$.
\begin{figure}
\centering

    {%
        \includegraphics[width = .48\linewidth]{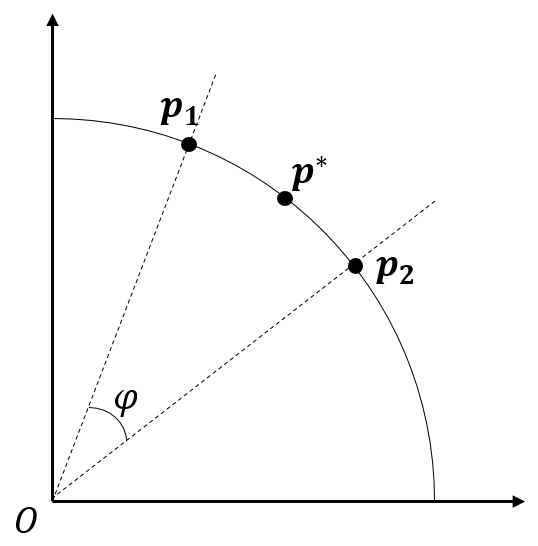}
        \includegraphics[width = .35\linewidth]{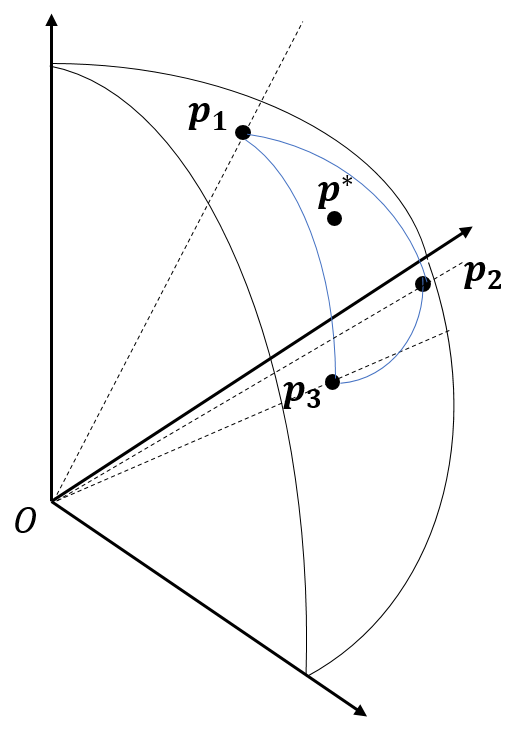}}

\caption{The example of the \textit{angle} of $n$ vectors in $n$-dimension space, where $n=2$ for the left and $n=3$ for the right.\label{fig:6}}
\end{figure}

Let's first consider the 2-dimension case (i.e. the left one of  Fig.6). Given two vectors $p_1$ and $p_2$, the probability can be measured by the ratio of the angle between $p_1$ and $p_2$ as $\frac{\phi}{\pi/2}$. For the cases where the dimension is higher than $n=2$ (i.e. the right one in Fig.6.),  \cite{ribando2006measuring} defines the normalized measurement of \textit{solid angle}
\begin{align}
    \tilde{V}_\Omega=\frac{\text{vol}_C(\Omega \cap B_C)}{\text{vol}_C(B_C)}=\frac{\text{vol}_{C-1}(\Omega\cap S_{C-1})}{\text{vol}_{C-1}S_{C-1}}
\end{align}
where $\text{vol}_C$ is the usual volume form in $\mathbb{R}^C$, $B_C$ is the unit ball in $\mathbb{R}^C$ and $S_C$ is the unit $C$-sphere. Further, they provide a way to compute the solid angle for $C$ unit positive vector 
\begin{align}
\tilde{V}_\Omega=\frac{1}{\pi^{C/2}}\int_{\mathbb{R}^C_{\ge 0}}e^{-\bold u^\top \bold P_t^\top \bold P_t\bold u}|\det\bold P_t|d\bold u
\end{align}
which represents the probability of the optimal weight in the convex cone by projecting $p^*$ and each ray (i.e. column) in $\bold P_t$ onto the surface of a unit sphere centered at $O$ in $\mathbb{R}^C$.

For the case where the number of sampled clients $M$ is smaller than $C$ or $M=C$ but $|\det \bold{P}_t|=0$, we have $\tilde{V}_\Omega=0$ (i.e. simply repeating one of the vectors in $P_t$ when $M<C$ until the number of vectors reaches $C$), which means it's nearly impossible to recover the uniformly distributed optimal weight by directly modifying the aggregation weights of the sampled clients. One way to approximate the global model update in this case is to choose clients with higher data quality in priority (i.e. local data distribution similar to the global one), which needs $\bold p^*$ and $\bold p_k$ are already known and is out of the scope of our discussion. Therefore, we limit our discussion to the case where $|\det \bold P_t|\ne 0$ and $M=C$ (i.e. the solid angle for $M>C$ can be computed by first dissecting $\bold P_t$ into simplicial cones \cite{ribando2006measuring}). 

Based on the equation (35), we demonstrate that the large average shortest-path distance will encourage a large $ \tilde{V}_\Omega$. We rearrang the equation (35) as
\begin{align}
    \tilde{V}_\Omega = \left(\frac{|\det \bold P_t|}{\pi^{C/2}}\right)\left(\int_{\mathbb{R}^C_{\ge 0}}e^{-\sum_{i,j\in[C]}u_iu_j\bold e_i^\top\bold e_j}d\bold u\right)
\end{align}

Now we respectively show how enlarging the average shortest-path distance (i.e. the equation (24)) leads to the increasing of the two terms on the right hand side of the equation (36).
\subsubsection{The Impact on The First Term.}
 To study how the determinant is impacted by the average shortest-path distance of the sampled clients, we compute the volume of the parallelepiped with the columns in $\bold P_t$ (i.e. Definition 1 and Definition 2) to obtain $|\det\bold P_t|$ according to Lemma 1 \cite{peng2007determinant}. 
 \begin{definition}
Let $\bold e_1,...,\bold e_C\in \mathbb{R}^C$. A parallelepiped $P=P(\bold e_1,...,e_C)$ is the set
\begin{align}
    P=\{\sum_{i=1}^Ct_i\bold e_i|0\le t_i\le 1\text{ for }i\text{  from }1 \text{ to } C\}
\end{align}
\end{definition}
\begin{definition}
The $n$-dimensional volume of a parallelepiped is 
\begin{align*}
\begin{split}
\text{Vol}_k[P(\bold e_1,...,\bold e_k)] = \left \{
\begin{array}{ll}
    \|\bold e_1\|,&k=1\\
    \text{Vol}_{k-1}[P(\bold e_1,...,\bold e_{k-1})]\|\bold{\tilde e}_k\|,&k>1
\end{array}
\right.
\end{split}
\end{align*}
where $\bold{\tilde e}_k=\bold e_k+(\bold e_1,...,\bold e_{k-1})\bold a=\bold e_k + \bold E_{k-1}\bold a_k$ and the unique chosen of $\bold a_k$ satisfies $\bold{\tilde e}_k^\top \bold e_i=0,\forall i\in[k-1]$.
\end{definition}
 \begin{lemma}
Given a $C$-dimensional parallelepiped $P$ in $C$-dimensional space defined by columns in $\bold P_t$, we have $\text{Vol}_C(P)=|\det \bold P_t|$
\end{lemma}

By denoting $P_k=P(\bold e_1,...,\bold e_k)$, the absolute value of the determinant of $\bold P_t$ is
 \begin{align}
     |\det \bold P_t|=\text{Vol}_C(P_C)=||\bold{\tilde e}_C||\text{Vol}_{C-1}(P_{C-1})\\=\sqrt{(\bold e_C + \bold E_{C-1}\bold a_C)^\top(\bold e_C + \bold E_{C-1}\bold a_C)}\text{Vol}_{C-1}(P_{C-1})
 \end{align}
 
Since $\bold{\tilde e}_C^\top \bold e_i=0,\forall i\in[C-1]$, we have
\begin{align}
    &\bold E_{C-1}^\top(\bold e_C+\bold E_{C-1}\bold a_C)=0\\\Rightarrow \bold a_C&=-( \bold E_{C-1}^\top \bold E_{C-1})^{-1}\bold E_{C-1}^\top\bold e_C
\end{align}
 Thus, the inner product of $\bold{\tilde e}_C$ with itself is
 \begin{align}
     \bold{\tilde e}_C^\top \bold{\tilde e}_C=\|\bold e_C-\bold E_{C-1}(\bold E_{C-1}^\top \bold E_{C-1})^{-1}\bold E_{C-1}^\top\bold e_C\|_2^2\\=1-[\bold e_C^\top \bold e_i]_{i\ne C}^\top(\bold E_{C-1}^\top \bold E_{C-1})^{-1}[\bold e_C^\top \bold e_i]_{i\ne C}
 \end{align}
 
 Given $|\det P|\ne0\rightarrow rank(\bold E_{C-1})=C-1$, we have that $\bold E_{C-1}^\top \bold E_{C-1}$ is a real symmetric matrix and full rank, which indicates that it is orthogonally diagonalizable. Thus, we orthogonally diagonalize $\bold E_{C-1}^\top \bold E_{C-1}$ into $Q\Lambda Q^\top$ to obtain its inverse $(\bold E_{C-1}^\top \bold E_{C-1})^{-1}=Q\Lambda^{-1} Q^\top$ where $\lambda_i>0$. And we rewrite the equation (43) with $\bold h_C=[\bold e_C^\top \bold e_1,...,\bold e_C^\top \bold e_{C-1}]$ as:
 \begin{align}
     1-\bold h_C^\top \tilde{Q}\tilde{Q}^\top\bold h_C&=1-\|\tilde{Q}\bold h_C\|_2^2\\&=1-\sum_{i=1}^{C-1} \frac{h_i^2}{\lambda_i^2}\\&=1-\sum_{i=1}^{C-1}\frac{1}{\lambda_i^2}(\bold e_C^\top \bold e_{i})^2
 \end{align}

From the equation (46), we can see that keeping $\bold e_C$ less similar with all the other vectors (i.e. $\bold e_C^\top \bold e_{i},\forall i\ne C$) will increase $|\det \bold P_t|$. In addition, the analysis doesn't specify the $C$th vector in $\bold P_t$ to be any particular client. Therefore, enlarging $f(S_t)$ will also enlarge the absolute value of the determinant of $\bold P_t$.

\subsubsection{The Impact on The Second Term.}

According to Cauchy–Schwarz inequality,
\begin{align}
    &\int_{\mathbb{R}^C_{\ge 0}}e^{-\sum_{i,j\in[C]}u_iu_j\bold e_i^\top\bold e_j}d\bold u\\\ge
    \int_{\mathbb{R}^C_{\ge 0}}&e^{-\sqrt{\sum_{i,j\in[C]}(u_iu_j)^2}\sqrt{\sum_{i,j\in[C]} (\bold e_i^\top\bold e_j)^2}}d\bold u\\\ge
    \int_{\mathbb{R}^C_{\ge 0}}&e^{-\sqrt{\sum_{i,j\in[C]}(u_iu_j)^2}\sqrt{\sum_{i,j\in[C]} \bold e_i^\top\bold e_j}}d\bold u\\=\int_{\mathbb{R}^C_{\ge 0}}&e^{-\sqrt{\sum_{i,j\in[C]}(u_iu_j)^2}\|\bold P_t^\top\bold P_t\|_{m1}^{\frac12}}d\bold u
\end{align}
which indicates that enlarging $f(S_t)$ in the equation (24) can improve the lower bound of the second term of the solid angle. 

Therefore, we conclude that keeping the average shortest-path distance of the sampled subset to be large will increase the chance to find a proper aggregation weight to well approximate the true model update.

\end{proof} 
\section{B. Results on Constructed 3DG}
We also run FedGS on the 3DG constructed by using the proposed functional similarity and cosine similarity of model parameters. We vary the same $\alpha\in\{0,0.5,1,2,5\}$ for $\text{FedGS}_{func}$ and $\text{FedGS}_{cos}$, and list the optimal results of them with the same settings of Table \ref{table2}, as is listed in Table 4. For Synthetic dataset, $\text{FedGS}_{cos}$ outperforms $\textbf{FedGS}$ and $\textbf{FedGS}_{func}$ under most of the client availability, which suggests that our definition of the oracle graph on Synthetic dataset is not the true oracle one. For CIFAR10, $\text{FedGS}_{func}$'s performance is competitive with the results obtained by \textbf{FedGS}, and the two methods consistently outperform $\textbf{FedGS}_{cos}$ in most cases. For FashionMNIST, $\text{FedGS}_{func}$ and $\text{FedGS}_{cos}$ suffer more performance reduction than CIFAR10. However, they still outperform MDSample and UniformSample when client availability changes (0.313 v.s. 0.333 in \textbf{YC}).
\begin{table}
\centering
\begin{tabular}{|l|l|l|l|l|} 
\hline
\multicolumn{2}{|l|}{\textbf{Setting}}                & \textbf{FedGS} & $\textbf{FedGS}_{func}$  & $\textbf{FedGS}_{cos}$\\ 
\hline
\multirow{5}{*}{\textbf{Synthetic}}    & \textbf{IDL} & 0.306& \textbf{0.304} &\textbf{0.304}\\ 
\cline{2-5}
                                       & \textbf{LN}  & 0.305  & 0.306 &\textbf{0.304}\\ 
\cline{2-5}
                                       & \textbf{SLN} & \textbf{0.318}  & 0.319&\textbf{0.318}\\ 
\cline{2-5}
                                       & \textbf{LDF} & 0.308  & \textbf{0.306} &0.307 \\ 
\cline{2-5}
                                       & \textbf{MDF} & 0.310 &  0.310 & \textbf{0.309} \\ 
\hline
\multirow{5}{*}{\textbf{CIFAR10}}      & \textbf{IDL} &0.968   &\textbf{0.963} &0.971 \\ 
\cline{2-5}
                                       & \textbf{LN}  & \textbf{0.972}  &\textbf{0.972} &0.973\\ 
\cline{2-5}
                                       & \textbf{SLN} & 0.996 & \textbf{0.991}  &  0.994\\ 
\cline{2-5}
                                       & \textbf{LDF} & \textbf{0.966} &0.975 &0.967 \\ 
\cline{2-5}
                                       & \textbf{MDF} & 0.963 & \textbf{0.962}  & 0.965\\ 
\hline
\multirow{3}{*}{\textbf{Fashion}} & \textbf{IDL} & \textbf{0.299}  & 0.305 & 0.304\\ 
\cline{2-5}
                                       & \textbf{YMF} &\textbf{0.303} & 0.307 & 0.307\\ 
\cline{2-5}
                                       & \textbf{YC}  & \textbf{0.307}  &0.314 & 0.313\\
\hline
\end{tabular}
\caption{The comparison of testing loss of FedGS running on the Oracle/Constructed 3DG.}
\label{table4}
\end{table}

\section{C. Experimental Details}
\subsection{Datasets}
\subsubsection{Synthetic.} We follow the setting in \cite{li2020federated} to generate this dataset by
\begin{align}
    y_{k,i}=\text{argmax}\{softmax(\bold W_k \bold x_{k,i}+\bold b_k)\}
\end{align}
where $(\bold x_{k,i}, y_{k,i})$ is the $i$th example in the local data $D_k$ of client $c_k$. For each client $c_k$, its local optimal model parameter $(\bold W_k,\bold b_k)$ is generated by $\mu_k\sim \mathcal{N}(0,\alpha)\in \mathbb{R},\bold W_{k}[i,j]\sim\mathcal{N}(\mu_k,1), \bold W_k\in \mathbb{R}^{10\times 60},\bold b_{k}[i]\sim\mathcal{N}(\mu_k,1), \bold b_k\in\mathbb{R}^{10}$, and its local data distribution is generated by $B_k\sim\mathcal{N}(0,\beta), \bold v_k[i]\sim\mathcal{N}(B_k,1), \bold v_k\in \mathbb{R}^{60}, \bold x_{k,i}\sim\mathcal{N}(\bold v_k,\bold \Sigma)\in \mathbb{R}^{60},\bold\Sigma=\text{diag}(\{i^{-1.2}\}_{i=1}^{60})$. The local data size for each client is $n_k\sim lognormal(4,2)$. In our experiments, we generate this dataset for 30 clients with $\alpha=\beta=0.5$.

\subsubsection{CIFAR10.} The CIFAR10 dataset \cite{krizhevsky2009learning} consists of totally 60000 32x32 colour images in 10 classes (i.e. 50000 training images and 10000 test images). We partition this dataset into 100 clients with both data size imbalance and data heterogeneity. To create the data imbalance, we set each client's local data size $n_k\sim lognormal(log(\frac{n}{N})-0.5, 1)$ to keep the mean data size is $\bar n=\frac{n}{N}$. Then, we generate the local label distribution $\bold p_k\sim Dirichlet(\alpha \bold p*)$ for each client, where $\bold p*$is the label distribution in the original dataset. Particularly, we loop replacing the local label distribution of clients with the new generated one from the same distribution until there exists no conflict with the allocated local data sizes, which allows the coexisting of controllable data imbalance and data heterogeneity. We use $\alpha=1.75$ in our experiments and provide the visualized partition in Fig.7(a).

\subsubsection{FashionMNIST.} The dataset \cite{xiao2017fashion} consists of a training set of 60,000 examples and a test set of 10,000 examples, where each example is a $28\times28$ size image of fashion and associated with a label from 10 classes. In this work, we partition this dataset into 100 clients and equally allocate the same number of examples to each one, where each client owns data from two labels according to \cite{mcmahan2017communication}. A direct visualization of the partitioned result is provided in Fig.7(b).

\begin{figure}
\centering
\subfigure[CIFAR10]{
	\label{fig7a}
	\includegraphics[width = .46\linewidth]{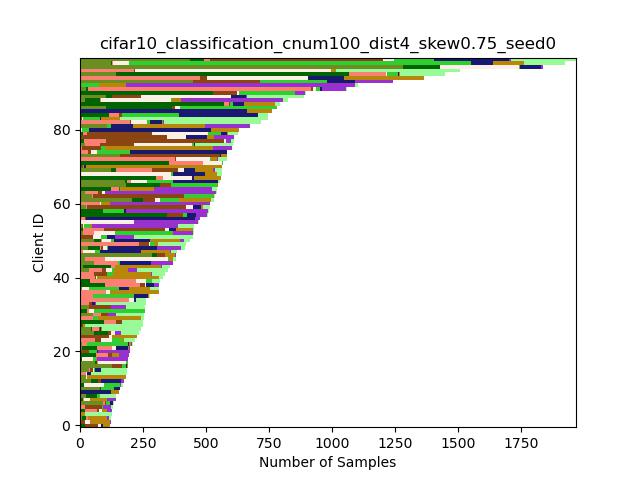}}
\subfigure[FashionMNIST]{
	\label{fig7b}
	\includegraphics[width = .46\linewidth]{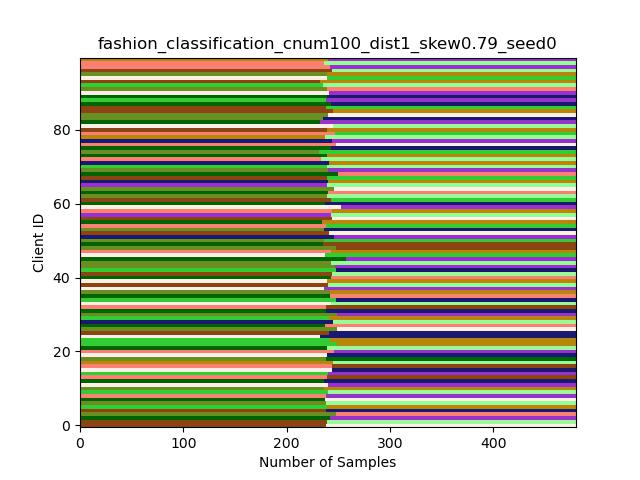}}
\caption{The visualization of data partition for CIFAR10 (a) and FashionMNIST (b). Each bar in the figures represents a client's local dataset and each label is assign to one color. The length of each bar reflects the size of the local data.}
\label{fig:7}
\end{figure}
\subsubsection{Models.}For CIFAR10 and FashionMNIST, we use CNNs that are respectively used for CIFAR10 and MNIST in \cite{mcmahan2017communication}. For Synthetic dataset, we use the same logistical regression model as \cite{li2020federated}.


\subsection{Hyperparameters}
For each dataset, we partition each client's local data into training and validating parts. Then, we tune the hyperparameters with FedAvg\cite{mcmahan2017communication} by grid search on the validation datasets under the ideal client availability. The batch size is $B=10$ for Synthetic and $B=32$ for both CIFAR10 and FashionMNIST. Specifically, we fixed the local updating steps instead of epochs to avoid unexpected bias caused by imbalanced data \cite{wang2020tackling}.  We search the number of local update steps in $E\in\{10,50,100\}$ for all the datasets and the learning rate $\eta_{Synthetic}\in\{0.01, 0.05, 0.1, 0.3\}, \eta_{CIFAR10, Fashion}\in\{0.003, 0.01, 0.03, 0.1\}$. For CIFAR10, the optimal parameters are $E=10, \eta=0.03$ and we train the model for 1000 rounds. For Synthetic, we train the model for 1000 round using $\eta=0.1,E=10$. For FashionMNIST, we train the model for 500 rounds with the optimal parameters $E=10,\eta=0.1$. We round-wisely decay the learning rate by a factor of $0.998$ for all the datasets.
To simulate the communication constraint of the server, we fixed the proportion of selected clients to be $0.1$ for CIFAR10 and FashionMNIST, and $0.2$ for Synthetic dataset, respectively.

\subsection{Client Availability Modes}

To obtain the results in Table \ref{table2}, we conduct experiments under the client availability modes of  \textbf{IDL}, \textbf{LN0.5}, \textbf{SLN0.5}, \textbf{LDF0.7}, \textbf{MDF0.7}, \textbf{YMF0.9} and \textbf{YC0.9} across different datasets, where the  float number at the end of the name of these settings is the coefficient $\beta$ that controls the degree of the global unavailability of clients.
For each client availability mode, we visualize the active states of clients in each communication round, which provides an intuitive way to distinguish the difference between the client availability modes. In Fig.8, we respectively visualize the \textbf{IDL}, \textbf{LDF0.7}, \textbf{MDF0.7} in Synthetic, \textbf{LN0.5}, \textbf{SLN0.5} in CIFAR10 and \textbf{YMF0.9}, \textbf{YC0.9} in FashionMNIST. For a fair comparison of different methods, we use an independent random seed (i.e. independent to the other random seeds used to optimizing the model) to control the active states of clients in each communication round, which promises that the active states of clients will remain unchanged when running different methods.

\begin{figure}
\centering

    {%
        \includegraphics[width = .47\linewidth]{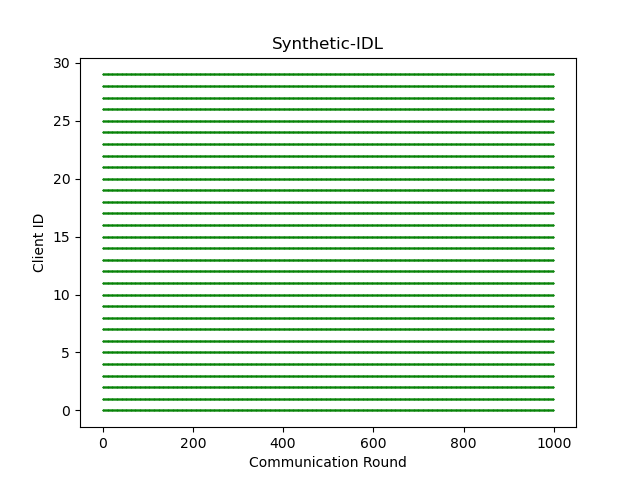}
        \includegraphics[width = .47\linewidth]{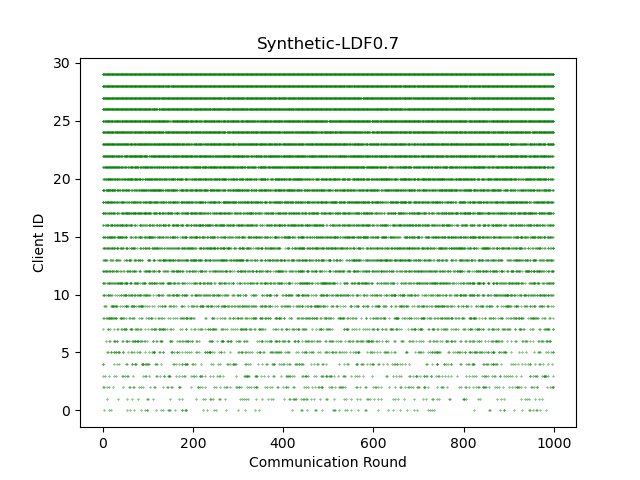}}

    {%
        \includegraphics[width = .47\linewidth]{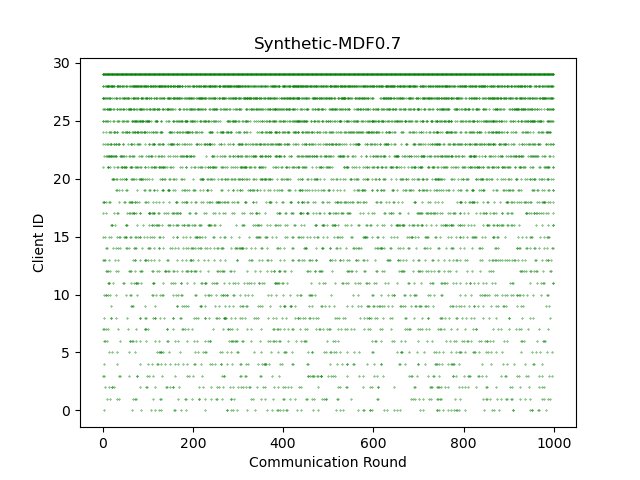}
        \includegraphics[width = .47\linewidth]{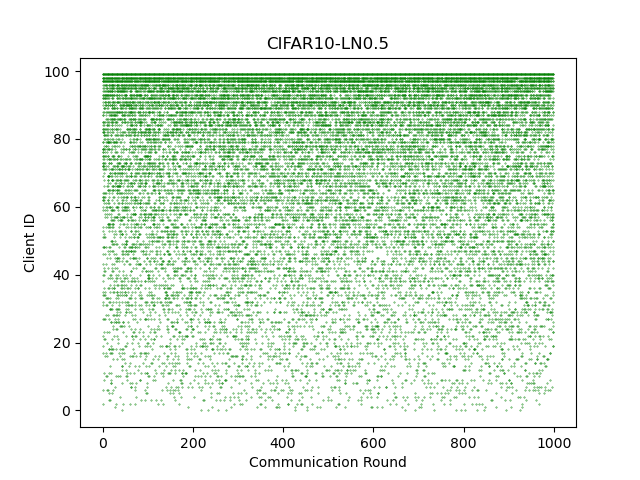}}
    {%
        \includegraphics[width = .47\linewidth]{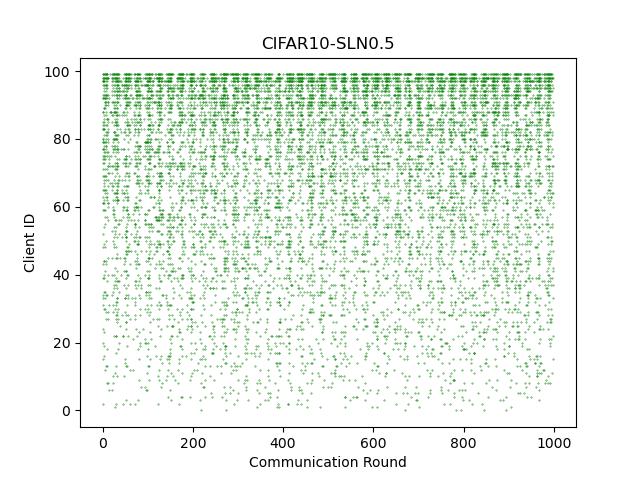}
        \includegraphics[width = .47\linewidth]{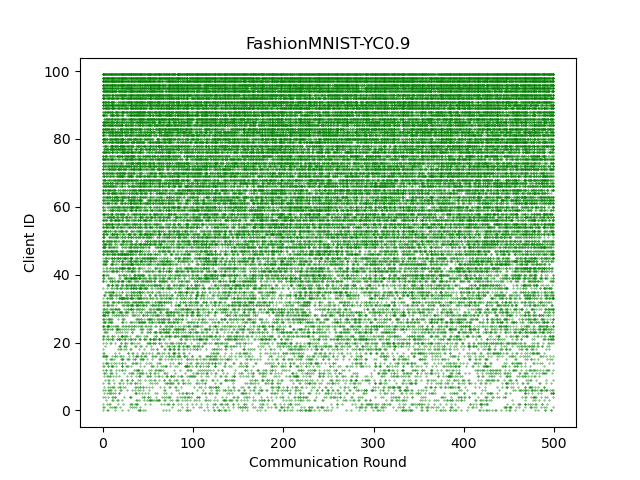}}

    {%
        \includegraphics[width = .47\linewidth]{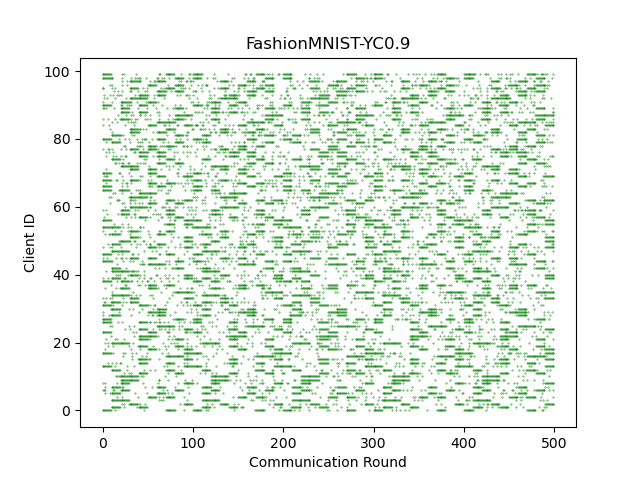}
        }

\caption{The active states of clients in each communication round under different availability modes.\label{fig:8}}
\end{figure}

\subsection{Oracle 3DG For Datasets}
For FashionMNIST and CIFAR10, we use the local data distribution vectors as the features to calculate the similarities among clients. For Synthetic, since the difference between clients' local data mainly locates in the optimal model, we directly use the local optimal model parameters $(\bold W_k,\bold b_k)$ as the feature vectors. Then, we use the inner dot as the similarity function to compute the similarity matrix $\bold V$. We further normalize the similarity value $V_{ij}$ to $[0,1]$ by computing $V_{ij}'=\frac{V_{ij}-\min_{i,j} V_{ij}}{\max_{i,j} V_{ij}-\min_{i,j} V_{ij}}$, and finally use the normalized similarity matrices to construst the oracle 3DGs as mentioned in Sec.3.2.

\section{D. Scalar Product Protocol}
In our first method to construct the 3DG, we use techniques based on Secure Scalar Product Protocols (SSPP) to compute the dot product of private vectors of parties. We argue that any existing solution of SSPP can be used in our settings. Here we introduce \cite{du2002building} to our FL scenario. The original protocol works as follows. Alice and Bob have different features on the same individuals and want to calculate the scalar product of their private vectors $\bold{A}$ and $\bold{B}$, both of size $m$ where $m$ is the sample size of the datasets. They will do this with the help of a commodity server named Merlin. The protocol consists of the following steps. First, Merlin generates two random vectors $\bold R_a $, $\bold R_b $ of size $m$ and two scalars $ r_a $ and $ r_b $ such that $ r_a + r_b = \bold R_a \cdot \bold R_b $, where either $ r_a $ or $ r_b $ is randomly generated. Merlin then sends $ \{ \bold R_a , r_a \} $ to Alice and $ \{ \bold R_b , r_b \} $ to Bob. Second, Alice sends $ \bold{\hat{A}} = \bold A + \bold R_a $ to Bob, and Bob sends $ \bold{\hat{B}} = \bold B + \bold R_b $ to Alice. Third, Bob generates a random number $ v_2 $ and computes $ u = \bold{\hat{A}} \cdot \bold{B} + r_b - v_2 $, then sends the result to Alice. Fourth, Alice computes $ u - (\bold{R_a} \cdot \bold{\hat{B}}) + r_a = \bold{A} \cdot \bold{B} - v_2 = v_1 $ and sends the result to Bob. Finally, Bob calculates the final result $ v_1 + v_2 = \bold{A} \cdot \bold{B} $. In our FL settings, since there are no communications between clients, we pass intermediate variables between clients through the server. The pseudo codes in Algorithm 2 summarizes the steps of scalar product in our FL settings.

\begin{algorithm}[tb]
    \caption{Scalar Product}
    \label{alg:2}

    \textbf{Input}:The client $A$'s feature vector $ \bold{A} \in \mathbb{R}^{d} $, and the client $B$'s feature vector $\bold{B} \in \mathbb{R}^{d}$\\
    \begin{algorithmic}[1]
    \STATE Server generates two random vectors $\bold R_a \in \mathbb{R}^{d}$, $\bold R_b \in \mathbb{R}^{d}$ and two scalars $ r_a $ and $ r_b $ such that $ r_a + r_b = \bold R_a \cdot \bold R_b $, where either $ r_a $ or $ r_b $ is randomly generated.
    \STATE Server sends $ \{ \bold R_a , r_a \} $ to client $A$ and $ \{ \bold R_b , r_b \} $ to client $B$.
    \STATE Client $A$ sends $ \bold{\hat{A}} = \bold A + \bold R_a $ to server, and client $B$ sends $ \bold{\hat{B}} = \bold B + \bold R_b $ to server.
    \STATE Server sends $ \bold{\hat{A}} $ to client $B$, and $ \bold{\hat{B}} $ to client $A$.
    \STATE Client $B$ generates a random number $ v_2 $ and computes $ u = \bold{\hat{A}} \cdot \bold{B} + r_b - v_2 $, then sends $u$ and $v_2$ to server.
    \STATE Server sends $u$ to client $A$.
    \STATE Client $A$ computes $ u - (\bold{R_a} \cdot \bold{\hat{B}}) + r_a = \bold{A} \cdot \bold{B} - v_2 = v_1 $ and sends $v_1$ to the server.
    \STATE Server calculates the scalar product $ v_1 + v_2 = \bold{A} \cdot \bold{B} $.
    \end{algorithmic}
    \end{algorithm}
\end{appendix}

\end{document}